\pdfoutput=1
\documentclass[11pt]{article}

\usepackage{EMNLP2022}
\usepackage{times}
\usepackage{latexsym}
\usepackage{multirow}
\usepackage{amsfonts,amssymb}
\usepackage{amsmath}
\usepackage{amsthm}
\usepackage{stmaryrd}
\usepackage{graphicx}
\usepackage{float}
\usepackage{color}

\definecolor{CornflowerBlue}{rgb}{0.39, 0.58, 0.93}
\definecolor{chocolate(web)}{rgb}{0.82, 0.41, 0.12}

\usepackage[T1]{fontenc}
\usepackage[utf8]{inputenc}

\usepackage{microtype}

\title{A Unified Positive-Unlabeled Learning Framework for Document-Level Relation Extraction with Different Levels of Labeling}

\author{Ye Wang\textsuperscript{1}, Xinxin Liu\textsuperscript{1}, Wenxin Hu\textsuperscript{1}\thanks{\enspace Corresponding author.}, Tao Zhang\textsuperscript{2} \\
  \textsuperscript{1}East China Normal University, Shanghai, China\\
  \textsuperscript{2}Tsinghua University, Beijing, China\\
  \texttt{\{yewang, xxliu\}@stu.ecnu.edu.cn, wxhu@cc.ecnu.edu.cn}\\
  \texttt{tao-zhan20@mails.tsinghua.edu.cn}}

\begin{document}
\maketitle
\begin{abstract}
Document-level relation extraction (RE) aims to identify relations between entities across multiple sentences. Most previous methods focused on document-level RE under full supervision. However, in real-world scenario, it is expensive and difficult to completely label all relations in a document because the number of entity pairs in document-level RE grows quadratically with the number of entities. \enspace To solve the common incomplete labeling problem, we propose a unified positive-unlabeled learning framework $-$ \underline{s}hift and \underline{s}quared \underline{r}anking loss \underline{p}ositive-\underline{u}nlabeled (SSR-PU) learning.
We use positive-unlabeled (PU) learning on document-level RE for the first time. Considering that labeled data of a dataset may lead to prior shift of unlabeled data, we introduce a PU learning under prior shift of training data. Also, using none-class score as an adaptive threshold, we propose squared ranking loss and prove its Bayesian consistency with multi-label ranking metrics. \enspace Extensive experiments demonstrate that our method achieves an improvement of about 14 F1 points relative to the previous baseline with incomplete labeling. In addition, it outperforms previous state-of-the-art results under both fully supervised and extremely unlabeled settings as well.\footnote{Code and data are available at \url{https://github.com/www-Ye/SSR-PU}.}
\end{abstract}

\section{Introduction}

Relation extraction (RE) aims to identify the relations between two entities in a given text. It has rich applications in knowledge graph construction, question answering, and biomedical text understanding. Most of the previous work was to extract relations between entities in a single sentence \citep{miwa-bansal-2016-end, zhang-etal-2018-graph}. Recently, document-level RE aiming to identify the relations \noindent among various entity pairs expressed in multiple sentences has received increasing research attention \citep{yao-etal-2019-docred, Zhou_Huang_Ma_Huang_2021,xu-etal-2022-document}.

% [H]
\begin{figure}
\centering
\includegraphics[width=0.5\textwidth]{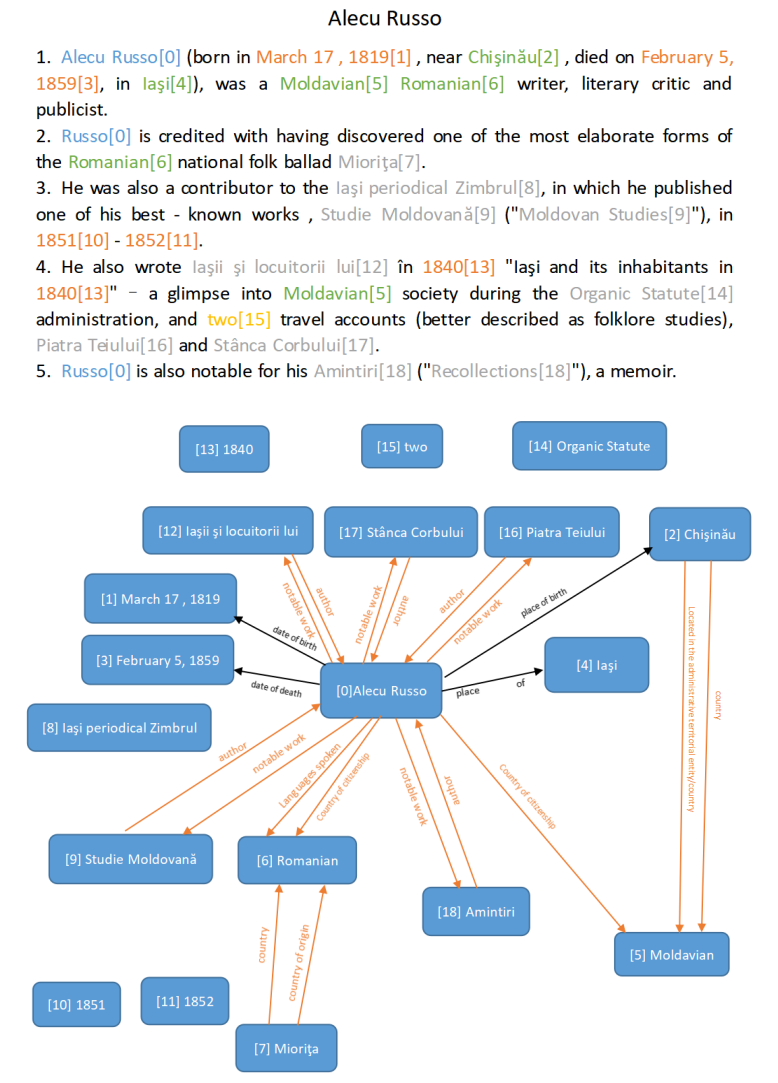} 
\caption{A case from DocRED. Entities are highlighted in different colors depending on their type. Black arrows indicate relations annotated with the original dataset, orange arrows indicate relations that are re-annotated by \citep{tan2022revisiting}.}
\label{fig1}
\end{figure}

Previous document-level RE methods mainly deal with fully supervised scenarios. However, in real-world scenarios, incomplete labeling is a common problem in document-level RE because the number of entity pairs grows quadratically with the number of entities. DocRED \citep{yao-etal-2019-docred} is a popular dataset for document-level RE. Recent studies \citep{huang-etal-2022-recommend, tan2022revisiting} found that DocRED, which annotates data with a \emph{recommend-revise} scheme, contains a large number of false negative samples, i.e., many positive samples being unlabeled. As shown in Figure \ref{fig1}, document \emph{Alecu Russo} contains a large number of unlabeled positive relations. Consequently, the models trained on this dataset tend to overfit in real scenarios and get lower recall. As a result, document-level RE with incomplete labeling has become an emergency need. 

To solve this problem, we propose a unified positive-unlabeled learning framework $-$ \underline{s}hift and \underline{s}quared \underline{r}anking loss \underline{p}ositive-\underline{u}nlabeled (SSR-PU) learning, which can be adapted to labeling under different levels. We use positive-unlabeled (PU) learning for the first time on the document-level RE task. Since document-level RE is a multi-label classification task, we apply a binary PU learning method for each class (one-vs-all), converting it to multi-label PU learning. In addition, according to our observations, a considerable portion of the relations in DocRED, a dataset annotated by \emph{recommend-revise} scheme, have already been annotated. This leads to the deviation between the prior distribution of the unlabeled data and the overall prior distribution. To address this problem, we introduce an adaptive PU learning under prior shift of training data that adjusts the model based on the estimated overall prior distribution and the labeled positive sample distribution to be similar to ordinary PN learning or ordinary PU learning. Here positive-negative (PN) learning means treating all unlabeled samples as negative samples.

Also, to distinguish between none-class and pre-defined classes, we propose a squared ranking loss for none-class ranking such that positive pre-defined labels are ranked higher than none-class label and negative pre-defined labels are ranked lower. This is an ideal multi-label surrogate loss metric, and we theoretically prove its Bayesian consistency with the multi-label ranking metric proposed by \citep{ijcai2022p630}. This loss function can be well adapted to PU learning.

We conduct extensive experiments on two multi-label document-level RE datasets with incomplete labeling, DocRED \citep{yao-etal-2019-docred} and ChemDisGene \citep{zhang-etal-2022-distant}, a newly proposed multi-labeled biomedical document-level RE dataset. Experimental results show that our method SSR-PU outperforms previous baseline that did not consider the labeling incompleteness phenomenon by about 14 F1 points. In addition, we perform fully supervised experiments, as well as experiments on an extremely unlabeled data that is newly constructed, in which the number of each relation type labeled in each document is limited to 1. Experiments under two complementary settings demonstrate the effectiveness of our method with different levels of labeling. The contributions of this paper are summarized as follows: 
\begin{itemize}
\item We propose a unified positive-unlabeled learning framework, SSR-PU, to adapt document-level RE with different levels of incomplete labeling.
\item We apply PU learning for the first time to the document-level RE task and introduce a PU learning under prior shift of training data that can reach a balance between ordinary PN learning and ordinary PU learning based on the estimated prior and labeling distribution.
\item We propose squared ranking loss, which effectively improves performance relative to other loss functions, and prove its Bayesian consistency with multi-label ranking metrics.
\item Our method achieves state-of-the-art results in a variety of settings and provides a robust baseline for document-level RE with incomplete labels.
\end{itemize}

\section{Related Work}

\textbf{Document-level relation extraction.} \enspace Previous generally effective methods for document-level RE are mainly graph-based models and transformer-based models. Graph-based models \citep{nan-etal-2020-reasoning,li-etal-2020-graph,zeng-etal-2020-double, zeng-etal-2021-sire, Xu_Chen_Zhao_2021} gather entity information for relational inference with graph neural networks, and transformer-based methods \citep{Zhou_Huang_Ma_Huang_2021, Xu_Wang_Lyu_Zhu_Mao_2021,ijcai2021p551,tan-etal-2022-document} implicitly capture long-range dependencies. Recently, \citep{huang-etal-2022-recommend, tan2022revisiting} found that a large number of positive relations remain unlabeled in document-level RE datasets, especially unpopular relations. However, the previous methods did not consider unlabeled data separately. They simply treated them all as negative samples, which led to a lower recall and a significant drop in performance in realistic scenarios.

\textbf{PU learning.} \enspace Positive-unlabeled (PU) learning \citep{10.1145/1401890.1401920,NIPS2014_35051070,pmlr-v37-plessis15,NIPS2017_7cce53cf} aims to learn a classifier from positive and unlabeled data. PU learning is a kind of semi-supervised learning but there is a fundamental difference between them: while semi-supervised learning requires labeled negative data, PU learning requires only labeled positive data. Many current PU learning methods rely on an overall prior estimate, while some recent studies \citep{charoenphakdee2019positive, nakajima2021positive} have noticed a prior shift between the training set and the test set. \enspace On the other hand, PU learning has been used in many NLP applications, e.g., text classification \citep{10.5555/1630659.1630746}, sentence embedding \citep{cao-etal-2021-pause}, named entity recognition \citep{peng-etal-2019-distantly, zhou-etal-2022-distantly}, knowledge graph completion \citep{ijcai2022p312} and sentence-level RE \citep{He_Chen_Wang_Zhang_Wang_Zhang_2020}. However, this method is rarely applied to the document-level RE task.

\textbf{Multi-label classification.} \enspace Multi-label classification is a widely investigated problem, and here we focus on the loss function. Binary cross entropy (BCE) is the most popular multi-label loss, reducing the multi-label problem to a number of independent binary (one-vs-all) classification tasks. Recently, \citep{hui2020evaluation} have found that squared loss can also achieve better results in classification tasks. Another common multi-label loss function is pairwise ranking loss, which transforms multi-label learning into a ranking problem via pairwise (one-vs-one) comparison \citep{furnkranz2008multilabel, Li_2017_CVPR}. \enspace 
For multi-label PU learning, \citep{Kanehira_2016_CVPR} treated it as a multi-label PU ranking problem, and \citep{9724274} applied PU learning to multi-label common vulnerabilities and exposure classification by using one-vs-all strategy. For document-level RE task, \citep{ijcai2022p630} proposed a none-class ranking multi-label metric. This multi-label metric has not yet been applied to PU learning.

\section{Methodology}

In this section, we introduce the details of our method \underline{s}hift and \underline{s}quared \underline{r}anking loss \underline{p}ositive-\underline{u}nlabeled (SSR-PU) learning for document-level RE with incomplete labeling. Firstly, we introduce the definition of positive-unlabeled learning for document-level RE. Next, we present the PU learning under prior shift of training data. Finally, squared ranking loss using the none-class score as an adaptive threshold is proposed.

\subsection{Positive-unlabeled learning for document-level RE}

Document-level RE can be viewed as a multi-label classification task, where each entity pair is an instance and the associated relations are label samples. Previous supervised learning methods only treated unlabeled relations as negative samples, which may lead to low recall in the presence of a large number of false negatives. To address this problem, we adopt PU learning \citep{NIPS2014_35051070,pmlr-v37-plessis15} for each class.

Let $\mathcal{X}$ be an instance space and $\mathcal{Y}=\{-1,+1\}^{K}$ be a label space, where $K$ is the number of pre-defined classes. An instance $\boldsymbol{x} \in \mathcal{X}$ is associated with a subset of labels, identified by a binary vector $\boldsymbol{y} \in \mathcal{Y}=\left(y_{1}, \ldots, y_{K}\right)$, where $y_{i}=+1$ if the $i$-th label is positive for $\boldsymbol{x}$, and $y_{i}=-1$ otherwise. A score function is defined as $\boldsymbol{f}(\boldsymbol{x})=(f_{1}(\boldsymbol{x}),f_{2}(\boldsymbol{x}),...,f_{K}(\boldsymbol{x}))$. In the following we use $f_{i}$ instead, to omit the dependency on x. 

For $i$-th class, assume that the data follow an unknown probability distribution with density $p(\boldsymbol{x}, y_{i})$, $p_{\mathrm{P}_{i}}=p(\boldsymbol{x} \mid y_{i}=+1)$ as the positive marginal, $p_{\mathrm{N}_{i}}=p(\boldsymbol{x} \mid y_{i}=-1)$ as the negative marginal, and $p_{i}(\boldsymbol{x})$ as the marginal. In positive-negative (PN) learning, the goal is to minimize the expected classification risk:
\begin{equation}
\begin{aligned}\label{eq1}
R_{\mathrm{PN}}(f)=\sum_{i=1}^{K}\mathbb{E}_{\boldsymbol{x},y_{i} \sim p(\boldsymbol{x}, y_{i})}[\ell(f_{i}, y_{i})],
\end{aligned}
\end{equation}
Here, Eq.\ref{eq1} can be calculated by equivalently using the sum of the errors of positive and negative samples:
\begin{equation}
\begin{aligned}\label{eq2}
R_{\mathrm{PN}}(f)&=\sum_{i=1}^{K}(\pi_{i} \mathbb{E}_{\mathrm{P}_{i}}[\ell(f_{i}, +1)]\\&+(1-\pi_{i})\mathbb{E}_{\mathrm{N}_{i}}[\ell(f_{i}, -1)]),
\end{aligned}
\end{equation}
where $\pi_{i}=p(y_{i}=+1)$ and $(1-\pi_{i})=(1-p(y_{i}=+1))=p(y_{i}=-1)$ is the positive and negative prior of the $i$-th class. $\mathbb{E}_{\mathrm{P}_{i}}[\cdot]=\mathbb{E}_{\boldsymbol{x} \sim p(\boldsymbol{x} \mid y_{i}=+1)}[\cdot]$, $\mathbb{E}_{\mathrm{N}_{i}}[\cdot]=\mathbb{E}_{\boldsymbol{x} \sim p(\boldsymbol{x} \mid y_{i}=-1)}[\cdot]$ and the loss function is represented by $\ell$. Rewriting Eq.\ref{eq2} into a form that uses the data for approximation, we get:

\begin{equation}
\begin{aligned}\label{eq3}
\widehat{R}_{\mathrm{PN}}(f)=&\sum_{i=1}^{K}(\frac{\pi_{i}}{n_{\mathrm{P}_{i}}} \sum_{j=1}^{n_{\mathrm{P}_{i}}} \ell(f_{i}( \boldsymbol{x}_{j}^{\mathrm{P}_{i}}), +1) \\
&+\frac{(1-\pi_{i})}{n_{\mathrm{N}_{i}}} \sum_{j=1}^{n_{\mathrm{N}_{i}}} \ell(f_{i}( \boldsymbol{x}_{j}^{\mathrm{N}_{i}}), -1)),
\end{aligned}
\end{equation}
where $\boldsymbol{x}_{j}^{\mathrm{P}_{i}}$ and $\boldsymbol{x}_{j}^{\mathrm{N}_{i}}$ denote cases that the $j$-th sample of class $i$ is positive or negative. $n_{\mathrm{P}_{i}}$ and $n_{\mathrm{N}_{i}}$ are the number of positive and negative samples of class $i$, respectively.

In positive-unlabeled (PU) learning, due to the absence of negative samples, we cannot  estimate $\mathbb{E}_{\mathrm{N}_{i}}[\cdot]$ from the data. Following \citep{NIPS2014_35051070}, PU learning assumes that unlabeled data can reflect the true overall distribution, that is, $p_{\mathrm{U}_{i}}(\boldsymbol{x})=p_{i}(\boldsymbol{x})$. The expected classification risk formulation can be defined as:
\begin{equation}
\begin{aligned}\label{eq4}
R_{\mathrm{PU}}(f)&=\sum_{i=1}^{K}(\pi_{i} \mathbb{E}_{\mathrm{P}_{i}}[\ell(f_{i}, +1)]+\\&\mathbb{E}_{\mathrm{U}_{i}}[\ell(f_{i}, -1)]-\pi_{i} \mathbb{E}_{\mathrm{P}_{i}}[\ell(f_{i}, -1)]),
\end{aligned}
\end{equation}
Here $\mathbb{E}_{\mathrm{U}_{i}}[\cdot]=\mathbb{E}_{\boldsymbol{x} \sim p_{i}(\boldsymbol{x})}[\cdot]$ and $\mathbb{E}_{\mathrm{U}_{i}}[\ell(f_{i}, -1)]$ $-\pi_{i} \mathbb{E}_{\mathrm{P}_{i}}[\ell(f_{i}, -1)]$ can alternatively represent $(1-\pi_{i})\mathbb{E}_{\mathrm{N}_{i}}[\ell(f_{i}, -1)]$ because $p_{i}(\boldsymbol{x})=
\pi_{i}p_{\mathrm{P}_{i}}(\boldsymbol{x})+(1-\pi_{i})p_{\mathrm{N}_{i}}(\boldsymbol{x})$.

By rewriting Eq.\ref{eq4} in a form that can be approximated using the data, we get the following:
\begin{equation}
\begin{aligned}\label{eq5}
\widehat{R}_{\mathrm{PU}}(f)=&\sum_{i=1}^{K}(\frac{\pi_{i}}{n_{\mathrm{P}_{i}}} \sum_{j=1}^{n_{\mathrm{P}_{i}}} \ell(f_{i}( \boldsymbol{x}_{j}^{\mathrm{P}_{i}}), +1) \\
&+[\frac{1}{n_{\mathrm{U}_{i}}} \sum_{j=1}^{n_{\mathrm{U}_{i}}} \ell(f_{i}( \boldsymbol{x}_{j}^{\mathrm{U}_{i}}), -1)\\&-\frac{\pi_{i}}{n_{\mathrm{P}_{i}}} \sum_{j=1}^{n_{\mathrm{P}_{i}}} \ell(f_{i}( \boldsymbol{x}_{j}^{\mathrm{P}_{i}}), -1)]),
\end{aligned}
\end{equation}
where $\boldsymbol{x}_{j}^{\mathrm{U}_{i}}$ denote cases that the $j$-th sample is unlabeled as class $i$ and $n_{\mathrm{U}_{i}}$ is the number of samples unlabeled as class $i$.

However, the second term in Eq.\ref{eq5} can be negative and can be prone to overfitting when using a highly flexible model. Thus, a non-negative risk estimator \citep{NIPS2017_7cce53cf} is proposed to alleviate the overfitting problem:
\begin{equation}
\begin{aligned}\label{eq7}
\widehat{R}_{\mathrm{PU}}&(f)=\sum_{i=1}^{K}(\frac{\pi_{i}}{n_{\mathrm{P}_{i}}} \sum_{j=1}^{n_{\mathrm{P}_{i}}} \ell(f_{i}( \boldsymbol{x}_{j}^{\mathrm{P}_{i}}), +1)+ \\
&\mathrm{max}(0,[\frac{1}{n_{\mathrm{U}_{i}}} \sum_{j=1}^{n_{\mathrm{U}_{i}}} \ell(f_{i}( \boldsymbol{x}_{j}^{\mathrm{U}_{i}}), -1)\\&-\frac{\pi_{i}}{n_{\mathrm{P}_{i}}} \sum_{j=1}^{n_{\mathrm{P}_{i}}} \ell(f_{i}( \boldsymbol{x}_{j}^{\mathrm{P}_{i}}), -1)])).
\end{aligned}
\end{equation}

For $\ell$, we use the convex function squared loss:
\begin{equation}
\begin{aligned}\label{eq8}
\ell(f_{i}, y_{i})=\frac{1}{4}(y_{i}f_{i}-1)^{2},
\end{aligned}
\end{equation}
and we compare the performance of using squared loss and log-sigmoid loss functions in Section \ref{c4}. The latter is a convex loss function commonly used in classification.

In addition, to solve the heavy class imbalance problem, we multiply $\gamma_{i}=(\frac{1-\pi_{i}}{\pi_{i}})^{0.5}$ before positive risk estimations as the class weight.

\subsection{Class prior shift of training data}
\begin{figure}
\centering
\includegraphics[width=0.4\textwidth]{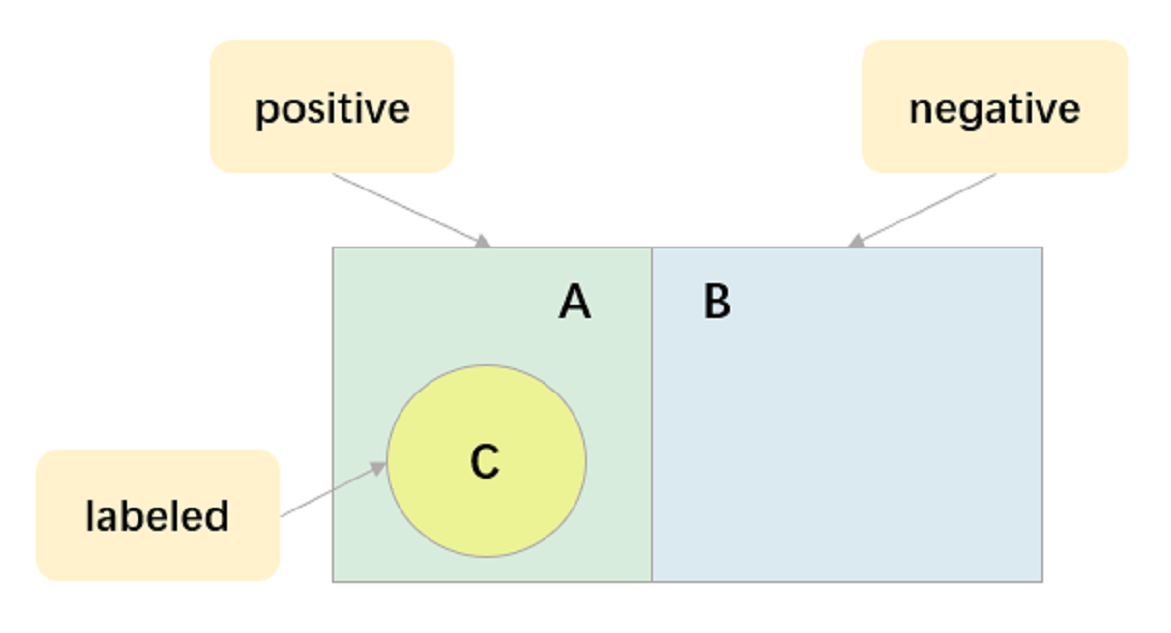} 
\caption{Positive sample distribution shift after labeled, i.e., $p(A \mid \overline{C}) \neq p(A)$}
\label{fig2}
\end{figure}
Ordinary PU learning requires an assumption that the overall distribution needs to be the same as the distribution of the unlabeled data. In contrast, with the document-level RE dataset constructed by a \emph{recommend-revise} scheme, many relations are probably already annotated, especially the common ones. This leads to a prior shift in the unlabeled data of the training set. When this assumption is broken, ordinary PU learning will yield a biased result. To address this problem, inspired by the method \citep{charoenphakdee2019positive} for handling a prior shift between the test set and the training set, we introduce the PU learning under prior shift of training data.

For each class, assume that the original prior $\pi_{i}=p(y_{i}=+1)$. We set $\pi_{labeled,i}=p(s_{i}=+1)$ and $(1-\pi_{labeled,i})=(1-p(s_{i}=+1))=p(s_{i}=-1)$ where $s_{i}=+1$ or $s_{i}=-1$ mean that the $i$-th class is labeled or unlabeled, respectively. As shown in Figure \ref{fig2}, the conditional probability of a positive sample under unlabeled data is different from the probability of an overall positive sample. The conditional probability of a positive sample under unlabeled data is:
\begin{equation}
\begin{aligned}\label{eq9}
p(y_{i}=1 \mid s_{i}=-1)=\frac{p(y_{i}=1, s_{i}=-1)}{p(s_{i}=-1)},
\end{aligned}
\end{equation}
where $p(y_{i}=1, s_{i}=-1)=\pi_{i}-\pi_{labeled,i}$, we can obtain the prior of positive samples in the new unlabeled data after labeling as $\pi_{u,i}=p(y_{i}=1 \mid s_{i}=-1)=\frac{\pi_{i}-\pi_{labeled,i}}{1-\pi_{labeled,i}}$. 

For document-level RE, the goal is to minimize the following misclassification risk for the original distribution of the training data:
\begin{equation}
\begin{aligned}\label{eq10}
R_{\mathrm{ori}}(f)&=\sum_{i=1}^{K}(\pi_{i} \mathbb{E}_{\mathrm{P}_{i}}[\ell(f_{i}, +1)]\\&+(1-\pi_{i})\mathbb{E}_{\mathrm{N}_{i}}[\ell(f_{i}, -1)]).
\end{aligned}
\end{equation}

We can express $R_{\mathrm{ori}}(f)$ using the expectation of positive and unlabeled data by the following theorem.
\newtheorem{thm}{\bf Theorem}
\begin{thm}\label{thm1}
The misclassification risk $R_{\mathrm{ori}}(f)$ can be equivalently expressed as
\begin{equation}
\begin{aligned}\label{eq11}
R_{\mathrm{S-PU}}(f)&=\sum_{i=1}^{K}(\pi_{i} \mathbb{E}_{\mathrm{P}_{i}}[ \ell(f_{i}, +1)] \\
&+\frac{1-\pi_{i}}{1-\pi_{u,i}} \mathbb{E}_{\mathrm{U}_{i}}[\ell(f_{i}, -1)] \\
&-\frac{\pi_{u,i}-\pi_{u,i} \pi_{i}}{1-\pi_{u,i}}\mathbb{E}_{\mathrm{P}_{i}}[\ell(f_{i}, -1)]).
\end{aligned}
\end{equation}
\end{thm} 
\begin{proof}
Proof appears in Appendix \ref{a0}.
\end{proof}

As a result, we can obtain the non-negative risk estimator \citep{NIPS2017_7cce53cf} under class prior shift of training data as follows:
\begin{equation}
\begin{aligned}\label{eq13}
&\widehat{R}_{\mathrm{S-PU}}(f)=\sum_{i=1}^{K}( \frac{1}{n_{\mathrm{P}_{i}}}\pi_{i} \sum_{j=1}^{n_{\mathrm{P}_{i}}}\ell(f_{i}(\boldsymbol{x}_{j}^{\mathrm{P}_{i}}), +1) \\&+\mathrm{max}(0, [\frac{1}{n_{\mathrm{U}_{i}}} \frac{1-\pi_{i}}{1-\pi_{u,i}} \sum_{j=1}^{n_{\mathrm{U}_{i}}} \ell(f_{i}(\boldsymbol{x}_{j}^{\mathrm{U}_{i}}), -1)\\&-\frac{1}{n_{\mathrm{P}_{i}}}\frac{\pi_{u,i}-\pi_{u,i} \pi_{i}}{1-\pi_{u,i}} \sum_{j=1}^{n_{\mathrm{P}_{i}}}\ell(f_{i}(\boldsymbol{x}_{j}^{\mathrm{P}_{i}}), -1)])).
\end{aligned}
\end{equation}

We can observe that PN learning and PU learning are special cases of this function. When $\pi_{u,i}=0$, this equation reduces to the form of ordinary PN learning, and when $\pi_{u,i}=\pi_{i}$, this equation reduces to the form of ordinary PU learning.

\subsection{Squared ranking loss}

To better measure the performance of document-level RE, \citep{ijcai2022p630} proposed a new multi-label performance measure:
\begin{equation}
\begin{aligned}\label{eq14}
L_{\mathrm{NA}}&(\boldsymbol{f}, \boldsymbol{y})= \sum_{i=1}^{K}(\llbracket y_{i}>0 \rrbracket \llbracket f_{i}<f_{0} \rrbracket\\
&+\llbracket y_{i} \leq 0 \rrbracket \llbracket f_{i}>f_{0} \rrbracket+\frac{1}{2} \llbracket f_{i}=f_{0} \rrbracket),
\end{aligned}
\end{equation}
where positive pre-defined labels should be ranked higher than the none-class label and negative ones should be ranked below. $\llbracket \cdot \rrbracket$ is an indicator function that takes the value of 1 when the conditions in the parentheses are met, otherwise 0.

However, it is difficult to optimize the above equation directly. Thus, we propose the squared ranking surrogate loss by rewriting Eq.\ref{eq8} as:
\begin{equation}
\begin{aligned}\label{eq15}
\ell_{\mathrm{SR}}(f_{i}, y_{i})=
\frac{1}{4}(y_{i}(f_{i}-f_{0})-margin)^{2},
\end{aligned}
\end{equation}
where $margin$ is a hyper-parameter and $f_{0}$ is the none-class score, when $f_{i}$ is greater than $f_{0}$ the label exists, and otherwise not.

Next we prove the Bayesian consistency of $\ell_{\mathrm{SR}}$ with the multi-label ranking metric $L_{\mathrm{NA}}$ when $margin \neq 0$. Given an instance x, let $\Delta_{i} = \mathrm{P}(y_{i}=1 \mid x)$ be the marginal probability when the $i$-th label is positive, the Bayes optimal score function $\boldsymbol{f}_{\mathrm{NA}}^{*}$ that minimizes the multi-label risk $\mathbb{E}[L_{\mathrm{NA}}(\mathrm{P}, \boldsymbol{f}) \mid \boldsymbol{x}]$ is given by:
\begin{equation}
\begin{aligned}\label{eq16}
\boldsymbol{f}_{\mathrm{NA}}^{*} \in \{\boldsymbol{f} : &f_{i}>f_{0} \; if \; \Delta_{i}>\frac{1}{2}, \\&and \; f_{i}<f_{0} \; if \; \Delta_{i}<\frac{1}{2}\}.
\end{aligned}
\end{equation}

The next theory guarantees that the classifier obtained by minimizing the surrogate loss $\ell_{\mathrm{SR}}$ converges to the classifier with the lowest multi-label risk, thus making it possible to achieve a better classification performance w.r.t. corresponding to the multi-label performance metric.

\begin{thm}\label{thm2}
$\ell_{\mathrm{SR}}$ (Eq.\ref{eq15}) is Bayes consistent w.r.t. $L_{\mathrm{NA}}$ (Eq.\ref{eq14}) when $margin \neq 0$.
\end{thm} 
\begin{proof}
Proof appears in Appendix \ref{a1}.
\end{proof}

As a supplement, we likewise compare the log-sigmoid ranking loss performance in Section \ref{c4}.

\begin{table}
\centering
\begin{tabular}{lcccc}
\hline \multirow{2}{*}{ Dataset } & \multicolumn{2}{c}{DocRED} & \multicolumn{2}{c}{ChemDisGene} \\
 & train & test & train & test \\
\hline 
\# docs & 3,053 & 500 & 76,942 & 523 \\
\# rels & \multicolumn{2}{c}{96} & \multicolumn{2}{c}{14} \\
Avg \# ents & $19.5$ & $19.6$ & $7.5$ & $10.0$ \\
Avg \# rels & $12.5$ & $34.9$ & $2.1$ & $7.2$ \\
\hline
\end{tabular}
\caption{ Statistics of Document-level RE Datasets}
\label{tab:accents1}
\end{table}

\begin{table*}
\centering
\begin{tabular}{lcccc}
\hline \textbf{Model} & \textbf{Ign F1} & \textbf{F1} & \textbf{P} & \textbf{R} \\
\hline BiLSTM$^{\ast}$ & $32.57 \pm 0.22$ & $32.86 \pm 0.22$ & $77.04 \pm 1.01$ & $20.89 \pm 0.17$ \\
GAIN+BERT$_{Base}^{\ast}$ & $45.57 \pm 1.36$ & $45.82 \pm 1.38$ & $88.11 \pm 1.07$ & $30.98 \pm 1.36$ \\
DocuNET+RoBERTa$_{Large}^{\ast}$ & $45.88 \pm 0.33$ & $45.99 \pm 0.33$ & $94.16 \pm 0.32$ & $30.42 \pm 0.29$ \\
\hline
\hline ATLOP+BERT$_{Base}^{\ast}$ & $43.12 \pm 0.24$ & $43.25 \pm 0.25$ & $\mathbf{92.49} \pm \mathbf{0.33}$ & $28.23 \pm 0.23$ \\
PN+ATLOP+BERT$_{Base}$ & $51.11 \pm 0.49$ & $51.68 \pm 0.40$ & $77.55 \pm 3.10$ & $38.79 \pm 0.49$ \\
SR-PN+ATLOP+BERT$_{Base}$ & $52.70 \pm 0.28$ & $53.10 \pm 0.26$ & $83.76 \pm 0.49$ & $38.87 \pm 0.23$ \\
PU+ATLOP+BERT$_{Base}$ & $51.80 \pm 1.11$ & $53.14 \pm 1.01$ & $58.81 \pm 2.41$ & $48.15 \pm 0.14$ \\
SR-PU+ATLOP+BERT$_{Base}$ & $53.87 \pm 0.27$ & $55.06 \pm 0.25$ & $63.42 \pm 0.64$ & $\mathbf{48.66} \pm \mathbf{0.11}$ \\
S-PU+ATLOP+BERT$_{Base}$ & $53.36 \pm 1.22$ & $54.44 \pm 1.12$ & $65.95 \pm 2.84$ & $46.38 \pm 0.22$ \\
SSR-PU+ATLOP+BERT$_{Base}$ & $\mathbf{55.21} \pm \mathbf{0.12}$ & $\mathbf{56.14} \pm \mathbf{0.12}$ & $70.42 \pm 0.18$ & $46.67 \pm 0.14$ \\
\hline ATLOP+RoBERTa$_{Large}^{\ast}$ & $45.09 \pm 0.26$ & $45.19 \pm 0.27$ & $\mathbf{94.75} \pm \mathbf{0.25}$ & $29.67 \pm 0.24$ \\
PN+ATLOP+RoBERTa$_{Large}$ & $54.21 \pm 0.34$ & $54.47 \pm 0.35$ & $89.22 \pm 0.36$ & $39.20 \pm 0.41$ \\
SR-PN+ATLOP+RoBERTa$_{Large}$ & $56.06 \pm 0.21$ & $56.39 \pm 0.23$ & $87.47 \pm 0.60$ & $41.61 \pm 0.39$ \\
PU+ATLOP+RoBERTa$_{Large}$ & $56.97 \pm 0.47$ & $58.04 \pm 0.43$ & $67.39 \pm 1.22$ & $50.98 \pm 0.39$ \\
SR-PU+ATLOP+RoBERTa$_{Large}$ & $57.64 \pm 0.25$ & $58.77 \pm 0.26$ & $66.39 \pm 0.47$ & $\mathbf{52.72} \pm \mathbf{0.44}$ \\
S-PU+ATLOP+RoBERTa$_{Large}$ & $58.19 \pm 0.24$ & $58.95 \pm 0.25$ & $75.68 \pm 0.36$ & $48.29 \pm 0.40$ \\
SSR-PU+ATLOP+RoBERTa$_{Large}$ & $\mathbf{58.68} \pm \mathbf{0.43}$ & $\mathbf{59.50} \pm \mathbf{0.45}$ & $74.21 \pm 0.53$ & $49.67 \pm 0.77$ \\
\hline
\end{tabular}
\caption{\label{table1}
Results on Re-DocRED revised test set. Results with $\ast$ are based on our implementation.
}
\end{table*}

\section{Experiments}
In this section, we evaluate our method on two multi-label document-level RE datasets with incomplete labeling. We also demonstrate the effectiveness of our method with different levels of labeling.

\subsection{Experimental Setups}

\textbf{Datasets.} \enspace \emph{DocRED} \citep{yao-etal-2019-docred} is a large-scale document-level RE dataset with 96 pre-defined relations constructed by a \emph{recommend-revise} scheme from Wikipedia. \citep{tan2022revisiting} observed a large number of false negatives in the annotation of DocRED and provided a high-quality revised version, Re-DocRED. In our experiments, we use the incompletely labeled DocRED original training set for training and the revised test set for testing. \enspace \emph{ChemDisGene} \citep{zhang-etal-2022-distant} is a newly proposed biomedical multi-label document-level RE dataset. This corpus is automatically derived from CTD database \citep{davis2021comparative} by distantly supervised method and has 523 abstracts labeled by domain experts as an additional \emph{All relationships} test set. We use the distantly supervised training set for training and the \emph{All relationships} test set for testing. \enspace The average number of relations per document in the test set on both two datasets is much larger than the average number of relations in the training set, which indicates the incomplete labeling phenomenon in the training set, with a large number of false negatives present. The statistics of the two datasets are listed in Table~\ref{tab:accents1}.

\textbf{Implementation details.} \enspace For each dataset, we use ATLOP \citep{Zhou_Huang_Ma_Huang_2021} as the encoding model for the representation learning of relations. Further, we apply cased $\mathrm{BERT}_{Base}$ \citep{devlin-etal-2019-bert} and $\mathrm{RoBERTa}_{Large}$ \citep{liu2019roberta} for DocRED and $\mathrm{PubmedBert}$ \citep{10.1145/3458754} for ChemDisGene. We use Huggingface's Transformers \citep{wolf-etal-2020-transformers} to implement all the models and AdamW \citep{loshchilov2018decoupled} as the optimizer, and apply a linear warmup \citep{goyal2017accurate} at the first 6\% steps followed by a linear decay to 0. For DocRED, we set the learning rates for $\mathrm{BERT}_{Base}$ and $\mathrm{RoBERTa}_{Large}$ settings to 5e-5 and 3e-5, respectively, in the same way as ATLOP. For ChemDisGene, the learning rate is set to 2e-5. The batch size (number of documents per batch) is set to 4 and 8 for two datasets, respectively. During our experiment, we set $\pi_{i}=3\pi_{labeled,i}$ and $margin=0.25$. To evaluate the efficacy of our methods in realistic settings, we do not use any fully labeled validation or test sets in any stage of the training process. The training stopping criteria are set as follows: 30 epochs for both two dataset. We report the performance of the final model instead of the best checkpoint. All experiments are conducted with 1 Tesla A100-40G GPU.

\begin{table*}
\centering
\begin{tabular}{lccc}
\hline \textbf{Model} & \textbf{F1} & \textbf{P} & \textbf{R} \\
\hline BRAN$^{\dagger}$ & $32.5$ & $41.8$ & $26.6$ \\
PubmedBert$^{\dagger}$ & $42.1$ & $64.3$ & $31.3$ \\
BRAN+PubmedBert$^{\dagger}$ & $43.8$ & $70.9$ & $31.6$ \\
\hline
\hline ATLOP+PubmedBert$^{\ast}$ & $42.73 \pm 0.36$ & $\mathbf{76.17} \pm \mathbf{0.54}$ & $29.70 \pm 0.36$ \\
PN+ATLOP+PubmedBert & $44.25 \pm 0.24$ & $73.46 \pm 0.95$ & $31.67 \pm 0.16$ \\
SR-PN+ATLOP+PubmedBert & $46.56 \pm 0.35$ & $69.84 \pm 0.54$ & $34.93 \pm 0.40$ \\
PU+ATLOP+PubmedBert & $44.60 \pm 0.70$ & $46.56 \pm 1.17$ & $42.80 \pm 0.35$ \\
SR-PU+ATLOP+PubmedBert & $45.86 \pm 0.38$ & $46.91 \pm 0.79$ & $\mathbf{44.86} \pm \mathbf{0.37}$ \\
S-PU+ATLOP+PubmedBert & $46.73 \pm 0.49$ & $53.95 \pm 1.14$ & $41.23 \pm 0.36$ \\
SSR-PU+ATLOP+PubmedBert & $\mathbf{48.56} \pm \mathbf{0.23}$ & $54.27 \pm 0.40$ & $43.93 \pm 0.32$ \\
\hline
\end{tabular}
\caption{\label{table4}
Results on ChemDisGene \emph{All relationships} test set. Results with $\dagger$ are reported from \citep{zhang-etal-2022-distant}. Results with $\ast$ are based on our implementation.
}
\end{table*}

\textbf{Baseline.} \enspace We re-implemented the existing fully supervised methods BiLSTM \citep{yao-etal-2019-docred}, GAIN \citep{zeng-etal-2020-double}, DocuNET \citep{ijcai2021p551} and ATLOP \citep{Zhou_Huang_Ma_Huang_2021} as the baseline models for DocRED in this new setup, where for GAIN and BiLSTM we use a fixed threshold of 0.5 and all methods take the final result of the model instead of the best checkpoint. For ChemDisGene, we used BRAN \citep{verga-etal-2018-simultaneously}, PubmedBert \citep{10.1145/3458754} and PubmedBert + BRAN mentioned in \citep{ zhang-etal-2022-distant} as the baseline models, and ATLOP is re-implemented as a supplementary baseline.

\textbf{Evaluation metric.} \enspace For DocRED, we use the micro F1 (F1), micro ignore F1 (Ign F1), precision (P) and recall (R) as the evaluation metrics to evaluate the overall performance of a model. Ign F1 measures the F1 score excluding the relations shared by the training and test set. \enspace For ChemDisGene, we use  micro F1 (F1), precision (P) and recall (R) as the evaluation metrics.

\subsection{Main Results}
In this subsection, we present the results of comparison of PN learning (PN), squared ranking loss PN learning (SR-PN), PU learning (PU), squared ranking loss PU learning (SR-PU), PU learning under prior shift of training data (S-PU) and SSR-PU. All methods use the same encoder and different loss functions. For each method, we use the same hyper-parameter settings and report the mean and standard deviation on the test set by conducting 5 runs with different random seeds (62, 63, 64, 65, 66).

\textbf{Results on DocRED.} \enspace As shown in Table \ref{table1}, our SSR-PU method achieves a state-of-the-art F1 and Ign F1 in both $\mathrm{BERT}_{Base}$ and $\mathrm{RoBERTa}_{Large}$ settings and outperforms the original ATLOP by 13.58 and 14.52 F1 points, respectively. Meanwhile, consistent with the observation in the paper \citep{huang-etal-2022-recommend}, existing document-level RE methods under full supervision have a significant performance degradation in the incompletely labeled scenario.

The original ATLOP method has the highest precision (P) but low recall (R), which implies that supervised learning methods that simply treat unlabeled data as negative samples lack the generalization ability to extract instances of relations that are systematically missed in the dataset. PN learning uses an estimated prior, but will yield a biased result because there are still positive samples in the unlabeled data. While PU learning uses both unlabeled and labeled data to better estimate the expectation of negative samples, which results in a higher recall rate. In addition, ordinary PU methods without prior shift overestimate the content of positive samples in unlabeled data, which means that the model will tend to identify more samples as positive, i.e., higher recall, but also leads to more false-positive prediction results, i.e., lower precision. In contrast, the S-PU method with prior shift effectively mitigates this phenomenon by bringing the positive samples estimated by the model in the unlabeled data closer to their true distribution. For example, in experiments under the $\mathrm{BERT}_{Base}$ setting, there is a small decrease in recall of less than 2 percentage points, while the precision improves by about 7 percentage points, leading to an improvement in the final results. And this phenomenon is more evident in common relations as analyzed in Section \ref{c4}. Finally, applying squared ranking loss in PN learning, PU learning and S-PU learning can further improve the performance of the model, demonstrating the effectiveness of the method with none-class score as an adaptive threshold for document-level RE.

\textbf{Results on ChemDisGene.}
As shown in Table \ref{table4}, the improvement of our method agrees with the results on DocRED, reaching the state-of-the-art F1, which is 5.83 F1 points higher than the original ATLOP. Notice that the improvement on ChemDisGene is not as dramatic as that on DocRED. We argue that this may be due to the fact that some of the documents in the extra annotated \emph{All relationships} test set are from another corpus DrugProt \citep{miranda2021overview}, and that the annotation by human experts has a large deviation from the original training set distribution. This suggests that it is a challenging direction to make the document-level RE model more generalizable when it is difficult to estimate the true distribution of the test set.

\begin{table}
\centering
\begin{tabular}{lcc}
\hline \textbf{Model} & \textbf{Ign F1} & \textbf{F1} \\
\hline ATLOP+BERT$_{Base}^{\ast}$ & $72.70$ & $73.47$ \\
SSR-PU+BERT$_{Base}$ & $\mathbf{72.91}$ & $\mathbf{74.33}$ \\
\hline ATLOP+RoBERTa$_{Large}^{\ast}$ & $76.92$ & $77.58$ \\
DocuNET+RoBERTa$_{Large}^{\dagger}$ & $77.27$ & $77.92$ \\
KD-DocRE+RoBERTa$_{Large}^{\dagger}$ & $77.63$ & $78.35$ \\
SSR-PU+RoBERTa$_{Large}$ & $\mathbf{77.67}$ & $\mathbf{78.86}$ \\
\hline
\end{tabular}
\caption{\label{table2}
Results on Re-DocRED revised test set under the fully supervised setting. Results with $\dagger$ are reported from \citep{tan2022revisiting}. Results with $\ast$ are based on our implementation.
}
\end{table}

\begin{table}
\centering
\begin{tabular}{lcc}
\hline \textbf{Model} & \textbf{Ign F1} & \textbf{F1} \\
\hline ATLOP+BERT$_{Base}^{\ast}$ & $16.99$ & $17.01$ \\
SSR-PU+BERT$_{Base}$ & $\mathbf{46.47}$ & $\mathbf{47.24}$ \\
\hline ATLOP+RoBERTa$_{Large}^{\ast}$ & $17.29$ & $17.31$ \\
SSR-PU+RoBERTa$_{Large}$ & $\mathbf{48.98}$ & $\mathbf{49.74}$ \\
\hline
\end{tabular}
\caption{\label{table3}
Results on Re-DocRED revised test set under the extremely unlabeled setting. Results with $\ast$ are based on our implementation.
}
\end{table}

\subsection{Different Levels of Labeling}

\textbf{Fully supervised setting.} \enspace In this setting, we set $\pi_{i}=\pi_{labeled,i}$ and other hyper-parameters identically. As shown in Table \ref{table2}, we use the \citep{tan2022revisiting} revised Re-DocRED dataset in the same fully supervised setting to compare with the current state-of-the-art baseline models ATLOP \citep{Zhou_Huang_Ma_Huang_2021}, DocuNET \citep{ijcai2021p551} and KD-DocRE \citep{tan-etal-2022-document}. Our method achieves the same state-of-the-art results, demonstrating the effectiveness of our method with full labeling. The result with this setting can be seen as an upper bound for document-level RE with incomplete labeling. More details of the experiment are shown in Appendix \ref{c2}.

\textbf{Extremely unlabeled setting.} \enspace In this setting, we use the original training set of DocRED to construct an extremely unlabeled training set, i.e., the number of labels for each relation type in the document being limited to 1. The average number of relations in the processed documents is reduced to 5.4. We consider this a more difficult and challenging scenario. We set $\pi_{i}=12\pi_{labeled,i}$ and other hyper-parameters identically. As shown in Table \ref{table3}, traditional supervised learning methods fail, while our proposed SSR-PU method still yields a robust result. It is worth noting that since the labeled sample is only a fraction of the true positive sample, i.e., the biased distribution, which means $p(x \mid y_{i}=1)$ is not equal to $p(x \mid s_{i}=1)$, the first term in Eq.\ref{eq13} is actually a biased approximation to the first term in Eq.\ref{eq11}. We consider this bias as one of the bottlenecks of the current method and the main reason why the method degrades a lot in extremely unlabeled scenarios, i.e., the bias is widened in extremely unlabeled scenarios. This is a good direction for future research, where possible solutions might involve adding some data augmentation or bootstrapping methods for labeling to alleviate this bias. More details of the experiment are shown in Appendix \ref{c3}.

\begin{table}
\centering
\begin{tabular}{lccc}
\hline \textbf{Model} & \textbf{Freq. F1} & \textbf{Freq. P} & \textbf{Freq. R} \\
\hline
SR-PN & $60.79$ & $87.83$ & $46.49$ \\
SR-PU & $62.43$ & $60.28$ & $64.74$ \\
SSR-PU & $64.88$ & $68.36$ & $61.74$ \\
\hline
\end{tabular}
\caption{\label{table_common1}
Results for the 10 most common relation types on Re-DocRED test set under the $\mathrm{BERT}_{Base}$ setting.
}
\end{table}

\begin{table}
\centering
\begin{tabular}{lccc}
\hline \textbf{Model} & \textbf{Freq. F1} & \textbf{Freq. P} & \textbf{Freq. R} \\
\hline
SR-PN & $47.62$ & $71.76$ & $35.64$ \\
SR-PU & $47.65$ & $44.35$ & $51.48$ \\
SSR-PU & $50.91$ & $52.09$ & $49.78$ \\
\hline
\end{tabular}
\caption{\label{table_common2}
Results for the 5 most common relation types on ChemDisGene \emph{All relationships} test set. 
}
\end{table}

\subsection{Additional Analysis}

\textbf{Analysis of common relations.} \enspace As shown in table \ref{table_common1} and table \ref{table_common2}, we show the results for common relations on DocRED and ChemDisGene, these frequent relation types account for about 60\% of the relation triples \citep{tan2022revisiting,zhang-etal-2022-distant}. It can be seen that the SR-PU method has a slightly higher recall and much lower precision, which corresponds to an overestimation of the positive sample size in the unlabeled data. The SSR-PU method, on the other hand, can alleviate this problem well, contributing to a better balance among precision and recall and better performance. This indicates a large amount of prior shift in common relations, which is consistent with \citep{huang-etal-2022-recommend} observation that common relations are more likely to be labeled in the dataset.

\textbf{Comparison with other loss functions.} \label{c4} \enspace We compare the squared loss with the log-sigmoid loss, which is commonly used in multi-label classification at the document-level RE. And again, this loss function is rewritten into a none-class ranking form for further comparison with squared ranking loss. The details of the loss function are listed in Appendix \ref{appendix5}. As shown in Table \ref{table6}, both the squared loss function and the squared ranking loss function are significantly improved compared to the other loss functions, which demonstrates the effectiveness of our proposed loss function in the multi-label document-level RE task.

\section{Conclusion and Future Work}

\begin{table}
\centering
\begin{tabular}{lcc}
\hline \textbf{Model} & \textbf{Ign F1} & \textbf{F1} \\
\hline
S-PU$_{log-sigmoid}$ & $52.23$ & $53.43$ \\
S-PU$_{squared}$ & $54.00$ & $55.01$ \\
\hline
S-PU$_{log-sigmoid \; ranking}$ & $52.42$ & $53.66$ \\
SSR-PU & $55.43$ & $56.36$ \\
\hline
\end{tabular}
\caption{\label{table6}
Results on Re-DocRED test set under the $\mathrm{BERT}_{Base}$ setting with different loss functions.
}
\end{table}

In this paper, we propose a unified positive-unlabeled learning framework, SSR-PU, which can effectively solve the incomplete labeling of document-level RE. We use PU learning on document-level RE for the first time and introduce a PU learning under prior shift of training data to adapt to different levels of labeling. Also, we propose squared ranking loss, using none-class score as an adaptive threshold. Experiments demonstrate that our method achieves state-of-the-art results with different levels of labeling and provides a robust new baseline for incompletely labeled document-level RE. \enspace In the future, we will consider methods that do not require estimation of priors, allowing generalization to unknown distributions more accurately, as well as addressing the problem of biased distributions with incomplete labeled positive samples and further improving the extraction performance of long-tail relations.

\section*{Limitations}
Regarding the limitations of our proposed method, our method requires an estimation of an overall prior that will affect the final result. In a realistic scenario, a very accurate prior estimation may be difficult to obtain. In addition, the biased distribution caused by the incomplete labeling of positive samples is one of the bottlenecks of the current method, and there is still much left to be improved for extremely unlabeled scenarios and scenarios where the gap between the test set and the training set distribution is too large, which can be a direction for further research. However, for now, we believe that our task is a valuable contribution to advancing the application of document-level RE in more realistic scenarios and provides a robust baseline for this direction.

\section*{Acknowledgements}
We sincerely thank all anonymous reviewers
for their valuable comments to improve our work. This research is funded by the Basic Research Project of Shanghai Science and Technology Commission (No.19JC1410101). The computation is supported by ECNU Multifunctional Platform for Innovation (001).

% Entries for the entire Anthology, followed by custom entries
\bibliography{anthology,custom}
\bibliographystyle{acl_natbib}

% \clearpage

\appendix

\section{Appendix}
\label{sec:appendix}

\begin{table*}
\centering
\begin{tabular}{lcccc}
\hline \textbf{Model} & \multicolumn{2}{c}{\textbf{Dev}} & \multicolumn{2}{c}{\textbf{Test}} \\
& \textbf{Ign F1} & \textbf{F1} & \textbf{Ign F1} & \textbf{F1} \\
\hline ATLOP+BERT$_{Base}^{\ast}$ & $73.12 \pm 0.35$ & $73.93 \pm 0.38$ & $72.70 \pm 0.23$ & $73.47 \pm 0.25$ \\
SSR-PU+BERT$_{Base}$ & $\mathbf{73.27} \pm \mathbf{0.19}$ & $\mathbf{74.69} \pm \mathbf{0.20}$ & $\mathbf{72.91} \pm \mathbf{0.23}$ & $\mathbf{74.33} \pm \mathbf{0.20}$ \\
\hline ATLOP+RoBERTa$_{Large}^{\ast}$ & $76.98 \pm 0.20$ & $77.68 \pm 0.21$ & $76.92 \pm 0.15$ & $77.58 \pm 0.16$ \\
DocuNET+RoBERTa$_{Large}^{\dagger}$ & $77.53$ & $78.16$ & $77.27$ & $77.92$ \\
KD-DocRE+RoBERTa$_{Large}^{\dagger}$ & $\mathbf{77.92}$ & $78.65$ & $77.63$ & $78.35$ \\
SSR-PU+RoBERTa$_{Large}$ & $77.44 \pm 0.25$ & $\mathbf{78.66} \pm \mathbf{0.23}$ & $\mathbf{77.67} \pm \mathbf{0.25}$ & $\mathbf{78.86} \pm \mathbf{0.23}$ \\
\hline
\end{tabular}
\caption{\label{table_full}
Results on revised Re-DocRED under the fully supervised setting. Results with $\dagger$ are reported from \citep{tan2022revisiting}. Results with $\ast$ are based on our implementation.
}
\end{table*}

\subsection{Proof of Theorem 1}
\label{a0}
\begin{proof}
Based on the fact that $p_{\mathrm{U}_{i}}(\boldsymbol{x})=
\pi_{u,i}p_{\mathrm{P}_{i}}(\boldsymbol{x})$ $+(1-\pi_{u,i})p_{\mathrm{N}_{i}}(\boldsymbol{x})$, $(1-\pi_{u,i}) \mathbb{E}_{\mathrm{N}_{i}}[\ell(f_{i}, -1)]$ can be alternatively expressed as $\mathbb{E}_{\mathrm{U}_{i}}[\ell(f_{i}, -1)]-\pi_{u,i} \mathbb{E}_{\mathrm{P}_{i}}[\ell(f_{i}, -1)]$. We can rewrite $R_{\mathrm{ori}}(f)$ as follows:
\begin{equation}
\begin{aligned}\label{eq12}
R_{\mathrm{ori}}(f)=& \sum_{i=1}^{K}(\pi_{i} \mathbb{E}_{\mathrm{P}_{i}}[\ell(f_{i}, +1)] \\
&+(1-\pi_{i}) \mathbb{E}_{\mathrm{N}_{i}}[\ell(f_{i}, -1)]) \\
=& \sum_{i=1}^{K}(\pi_{i} \mathbb{E}_{\mathrm{P}_{i}}[\ell(f_{i}, +1)] \\
&+\frac{1-\pi_{i}}{1-\pi_{u,i}}(\mathbb{E}_{\mathrm{U}_{i}}[\ell(f_{i}, -1)]\\&-\pi_{u,i} \mathbb{E}_{\mathrm{P}_{i}}[\ell(f_{i}, -1)])) \\
=& R_{\mathrm{S-PU}}(f).
\end{aligned}
\end{equation}
We conclude that $R_{\mathrm{ori}}(f)=R_{\mathrm{S-PU}}(f)$.
\end{proof}

\subsection{Proof of Theorem 2}
\label{a1}
\begin{proof}
Let $\Delta_{i} = \mathrm{P}(y_{i}=1 \mid x)$ be the marginal probability when the i-th label is positive. The conditional risk of $\ell_{\mathrm{SR}}$ is:
\begin{equation}
\begin{aligned}\label{eq17}
R_{\ell_{\mathrm{SR}}}&(\mathrm{P}, \boldsymbol{f})=\sum_{i=1}^{K}(\Delta_{i}\frac{1}{4}((f_{i}-f_{0})-margin)^{2}\\&+(1-\Delta_{i})\frac{1}{4}(-(f_{i}-f_{0})-margin)^{2}).
\end{aligned}
\end{equation}
For $i=1,...,K$, the partial derivative can be computed by
\begin{equation}
\begin{aligned}\label{eq18}
\frac{\partial}{f_{i}}\mathbb{E}&[{\ell_{\mathrm{SR}}}(\mathrm{P}, \boldsymbol{f}) \mid \boldsymbol{x}]=\\&\sum_{i=1}^{K}(\Delta_{i}\frac{1}{2}((f_{i}-f_{0})-margin)+\\&(1-\Delta_{i})\frac{1}{2}((f_{0}-f_{i})-margin)),
\end{aligned}
\end{equation}
since $\ell_{\mathrm{SR}}$ is convex and differentiable, we can obtain the optimal $\boldsymbol{f}^{*}$ by setting the partial derivatives to zero, which leads to
\begin{equation}
\begin{aligned}\label{eq19}
f_{i}^{*}-f_{0}^{*}=2\Delta_{i} margin-margin, i=1,...,K.
\end{aligned}
\end{equation}
When $margin \neq 0$, for the optimal score function $\boldsymbol{f}^{*}$, $f_{i}^{*}>f_{0}^{*}$ if and only if $\Delta_{i}>\frac{1}{2}$, which minimizes the $\ell_{\mathrm{SR}}$ risk according to Eq.\ref{eq16}. Therefore, $\ell_{\mathrm{SR}}$ is Bayes consistent w.r.t. $L_{\mathrm{NA}}$.
\end{proof}

\subsection{Results under the Fully Supervised Setting}
\label{c2}

The detailed results under the fully supervised setting are shown in Table \ref{table_full}. We report the mean and standard deviation on the validation and test set by conducting 5 runs with different random seeds (62, 63, 64, 65, 66).

\subsection{Results under the Extremely Unlabeled Setting}
\label{c3}

The detailed results under the extremely unlabeled setting are shown in Table \ref{table_ext}. We report the mean and standard deviation on the test set by conducting 5 runs with different random seeds (62, 63, 64, 65, 66).

\begin{table*}
\centering
\begin{tabular}{lcccc}
\hline \textbf{Model} & \textbf{Ign F1} & \textbf{F1} & \textbf{P} & \textbf{R} \\
\hline ATLOP+BERT$_{Base}^{\ast}$ & $16.99 \pm 0.24$ & $17.01 \pm 0.24$ & $\mathbf{93.17} \pm \mathbf{0.48}$ & $9.36 \pm 0.14$ \\
SSR-PU+BERT$_{Base}$ & $\mathbf{46.47} \pm \mathbf{0.21}$ & $\mathbf{47.24} \pm \mathbf{0.23}$ & $59.52 \pm 0.87$ & $\mathbf{39.18} \pm \mathbf{0.61}$ \\
\hline ATLOP+RoBERTa$_{Large}^{\ast}$ & $17.29 \pm 0.28$ & $17.31 \pm 0.28$ & $\mathbf{94.85} \pm \mathbf{0.19}$ & $9.52 \pm 0.17$ \\
SSR-PU+RoBERTa$_{Large}$ & $\mathbf{48.98} \pm \mathbf{0.30}$ & $\mathbf{49.74} \pm \mathbf{0.30}$ & $61.57 \pm 1.34$ & $\mathbf{41.75} \pm \mathbf{0.42}$ \\
\hline
\end{tabular}
\caption{\label{table_ext}
Results on Re-DocRED revised test set under the extremely unlabeled setting. Results with $\ast$ are based on our implementation.
}
\end{table*}

\subsection{Details of Other Loss Functions}
\label{appendix5}
We first show the convex loss function log-sigmoid loss, which is commonly used in classification task:

\begin{equation}
\begin{aligned}\label{eq20}
\ell_{LS}(f_{i}, y_{i})=-log(\sigma(y_{i}f_{i})),
\end{aligned}
\end{equation}
where $\sigma(x)$ is the sigmoid function.

Since log-sigmoid loss is convex and differentiable, we can obtain its none-class ranking form.

Log-sigmoid ranking loss:
\begin{equation}
\begin{aligned}\label{eq22}
\ell_{LSR}(f_{i}, y_{i})=-log(\sigma(y_{i}(f_{i}-f_{0}))).
\end{aligned}
\end{equation}

This ranking loss function remain Bayesian consistent with $L_{\mathrm{NA}}$ (Eq.\ref{eq14}).

\begin{table}
\centering
\begin{tabular}{lcc}
\hline \textbf{Model} & \textbf{Ign F1} & \textbf{F1} \\
\hline
SSR-PU$_{margin=0}$ & $0.18$ & $0.20$ \\
SSR-PU$_{margin=0.1}$ & $55.76$ & $56.81$ \\
SSR-PU$_{margin=0.25}$ & $55.43$ & $56.36$ \\
SSR-PU$_{margin=0.5}$ & $55.27$ & $56.19$ \\
SSR-PU$_{margin=1.0}$ & $54.25$ & $55.24$ \\
\hline
\end{tabular}
\caption{\label{table8}
Results on Re-DocRED revised test set under the $\mathrm{BERT}_{Base}$ setting with different $margin$.
}
\end{table}

\subsection{Sensitivity to Hyper-Parameter \emph{margin}}

As shown in Table \ref{table8}, the model fail to train when $margin=0$, and the model is insensitive to $margin$ when $margin \neq 0$. This is consistent with our proof.

\begin{table}
\centering
\begin{tabular}{lccc}
\hline \textbf{Model} & \textbf{F1} & \textbf{P} & \textbf{R} \\
\hline
SSR-PU$_{\pi_{i}=2\pi_{labeled,i}}$ & $55.44$ & $78.98$ & $42.71$ \\
SSR-PU$_{\pi_{i}=3\pi_{labeled,i}}$ & $56.36$ & $70.53$ & $46.93$ \\
SSR-PU$_{\pi_{i}=4\pi_{labeled,i}}$ & $54.74$ & $61.45$ & $49.35$ \\
\hline
\end{tabular}
\caption{\label{table9}
Results on Re-DocRED revised test set under the $\mathrm{BERT}_{Base}$ setting with different $\pi_{i}$ estimation.
}
\end{table}

\subsection{Influence of Prior Estimation}
As shown in Table \ref{table9}, the experimental results with different $\pi_{i}$ show that our method is insensitive to the estimation of $\pi_{i}$. Smaller estimates of $\pi_{i}$ lead to higher precision rates as well as lower recall rates, while the opposite is true for higher estimates of $\pi_{i}$.

\end{document}